\documentclass{article} 
                        
\usepackage[preprint]{neurips_2026}

\usepackage[american]{babel}

\usepackage{natbib} 

\usepackage{mathtools} 
\usepackage{booktabs} 
\usepackage{tikz} 

\usepackage[utf8]{inputenc}
\usepackage{amsmath,amsthm,amssymb,graphicx,mathtools,tikz,hyperref,bbm,listings,caption,subcaption}
\usepackage{notation}
\usepackage{researchpack}
\usepackage{algorithm}
\usepackage{algorithmic}
\usepackage[capitalise,nameinlink]{cleveref}
\usepackage{notation}
\usepackage{researchpack}
\usepackage{natbib}
\usetikzlibrary{positioning}
\allowdisplaybreaks[4]

\crefname{defn}{Def.}{Def.}
\crefname{section}{Sec.}{Sec.}
\crefname{algorithm}{Alg.}{Alg.} 
\crefname{thm}{Thm.}{Thm.}
\crefname{lem}{Lem.}{Lem.}
\crefname{prop}{Prop.}{Prop.}
\crefname{asm}{Asm.}{Asm.}
\crefname{appendix}{Appx.}{Appx.}

\newtheorem{theorem}{Theorem}[section]
\newtheorem{corollary}{Corollary}[theorem]
\newtheorem{lemma}[theorem]{Lemma}

\title{The Expressivity Boundary of Probabilistic Circuits: \\ A Comparison with Large Language Models}
\author{\textbf{Zhiyu Zhao}$^{1\ast}$, \textbf{Xuejie Liu}$^{2,3\ast}$, \textbf{Muhan Zhang}$^{2}$, \textbf{Anji Liu}$^{1}$\\
$^1$School of Computing, National University of Singapore \\ $^2$Institute for Artificial Intelligence, Peking University \\
$^3$School of Intelligence Science and Technology, Peking University \\
\tt zhiyu.zhao@nus.edu.sg, \tt xjliu@stu.pku.edu.cn, \\\tt muhan@pku.edu.cn, \tt anjiliu@comp.nus.edu.sg}

\begin{document}

\maketitle
\footnotetext[1]{Equal contribution.}

\begin{abstract}
Probabilistic Circuits (PCs) are deep generative models that support exact and efficient probabilistic inference. Yet in autoregressive language modeling, PCs still lag behind Transformer-based large language models (LLMs), suggesting an important expressivity gap.
In this work, we compare PCs and LLMs under a unified autoregressive formulation. First, an output bottleneck: PCs parameterize predictions as convex combinations in probability space, which struggles to represent the sharp distributions typical of language; adopting a logit-space parameterization substantially narrows this gap. Second, a context-encoding bottleneck: we prove that structured-decomposable PCs can match Transformer separation rank on vtree-aligned partitions, but show, both theoretically and empirically, that this capacity is limited to partitions aligned with the fixed routing structure, leading to severe degradation when the data exhibits heterogeneous dependency topologies. We further prove that decomposable PCs are strictly more expressive than structured-decomposable ones, though effectively optimizing them remains an open challenge.
\end{abstract}

\section{Introduction}
Generative models have achieved remarkable success in learning high-dimensional distributions and producing realistic samples across data modalities~\citep{kingma2013auto,goodfellow2014generative,ho2020denoising}.
Despite this progress, most state-of-the-art neural generative models are optimized primarily for sampling and scoring a narrow family of queries (e.g., next-token probabilities in autoregressive LLMs~\citep{radford2019language,brown2020language,touvron2023llama}), while many downstream settings require richer probabilistic reasoning such as exact marginalization, conditioning, and structured inference.

Probabilistic Circuits (PCs) \citep{choi2020probabilistic} are a class of deep generative models that occupy a distinct point in this landscape: by design, they support \emph{exact} and \emph{efficient} probabilistic queries over their learned distributions, including arbitrary marginal and conditional probability computations. This tractability offers a sharp contrast to dominant neural paradigms, where inference is often approximate or restricted to model-native queries~\citep{poon2011sum,vergari2021compositional,kisa2014probabilistic}.
As a result, PCs have enabled strong results in inference-demanding applications such as controlled generation~\citep{weng2025trace,zhang2023tractable}, lossless (or near-lossless) compression~\citep{liu2021lossless}, and reinforcement learning~\citep{liu2024tractable}, where repeated conditional queries and marginalizations are central to the task.
However, on standard generative modeling benchmarks, PCs still lag behind modern neural generators, suggesting an important open question: \emph{what are the root causes of the expressivity gap between tractable PCs and neural sequence models?}

Prior work has developed theoretical characterizations of PC expressiveness and succinctness, typically analyzing PCs in isolation or comparing subclasses within the tractable modeling family \citep{de2021compilation,martens2014expressive}.
In contrast, this paper takes a different route: we directly compare PCs against the dominant neural paradigm in sequence modeling, namely Transformer-based Large Language Models (LLMs).
To enable a head-to-head analysis, we cast both model classes under a unified autoregressive view as \emph{context-conditioned encoders}:
given a prefix $x_{<t}$, both models instantiate a mapping $\theta_t = \mathrm{Encoder}(x_{<t})$ that produces a predictive distribution over the next token $X_t$.
This unified lens allows us to investigate the expressivity gap through two structural bottlenecks: (i) the \emph{output bottleneck}, \ie how the internal representation is translated into a distribution over the vocabulary, and (ii) the \emph{context encoding bottleneck}, \ie how flexibly the model routes and fuses information from the prefix.

Our analysis reveals that the gap is not explained by a single factor, but instead decomposes into two mismatches between standard PCs and Transformers.

\textbf{1) The Output Mismatch.}
Transformers generate logits and apply a softmax, naturally accommodating the sparse, peaked next-token distributions of language. Standard PCs operate directly in probability space, and we find that adopting a logit-style output parameterization can substantially narrow the gap.



\textbf{2) The Routing Mismatch.}
Self-attention enables Transformers to route information dynamically based on the input. In contrast, structured PCs use a fixed routing pattern; we find that this rigidity can significantly hurt performance on data with heterogeneous dependency structures, yet when the data's dependency topology aligns with the circuit's structure, the gap can be markedly reduced, a finding that is not obvious a priori. We further explore relaxing structured decomposability to mitigate the fixed routing limitation of PCs.


Taken together, we characterize an expressivity gap between PCs and LLMs through a combination of theoretical and empirical analysis. Specifically, we attribute this gap to two structural bottlenecks, namely the output bottleneck and the context-encoding bottleneck. These findings also yield concrete implications for future work on improving PC expressiveness.

\section{Preliminaries}


In this section, we review the generative modeling frameworks of Transformer-based Large Language Models (LLMs) and Probabilistic Circuits (PCs). While both model classes aim to represent a joint distribution over a sequence of $T$ discrete (categorical) variables $\X = \{X_1, \dots, X_T\}$, they differ fundamentally in their computational structures and probabilistic semantics. We therefore first introduce the basic syntax and semantics of each model class, and then present a unified perspective in Section~\ref{sec:unified-decomposition} to support a direct comparison of their expressive power.

\paragraph{Large Language Models}

LLMs are typically defined as autoregressive models that decompose the joint distribution $p(\x)$ into a product of conditional distributions:
\begin{equation}
p(\x) = \prod_{t=1}^T p(x_t \given \x_{<t}), \label{eq:autoregressive_factorization}
\end{equation}
where $\x_{<t} = \{x_1, \dots, x_{t-1}\}$ denotes the prefix context. The core of a modern LLM is the Transformer backbone $f_{\mathrm{LLM}}$, which serves as a powerful context encoder. Specifically, for each token $\x_t$, the model computes a vector of parameters $\params_t$ (e.g., logits) that defines the categorical distribution over the vocabulary:
\begin{equation}
\params_t = f_{\mathrm{LLM}}(\x_{<t}), \quad p(x_t \given \x_{<t}) = \mathrm{Softmax}(\params_t).
\end{equation}
From a probabilistic perspective, the Transformer's expressiveness stems from its ability to learn complex dependencies in $f_{\mathrm{LLM}}$, allowing the conditional distribution $p(x_t \given \x_{<t})$ to depend on the entire history in a highly flexible manner.

\paragraph{Probabilistic Circuits}

Probabilistic Circuits \citep{choi2020probabilistic} represent a joint distribution through a rooted Directed Acyclic Graph (DAG). Each node $n$ in a PC explicitly represents a probability distribution $p_n$ over a subset of variables $\X_n \subseteq \X$, known as its scope. A PC defines this distribution by starting from input nodes and recursively combining them through inner nodes. Specifically, input nodes are the leaves of the DAG and define primitive distributions over their scope $\mathbf{X}_n$. Inner nodes represent distributions formed by their children and are of two types: product nodes, which represent the factorized distribution of their children, and sum nodes, which represent a weighted mixture of their children's distributions. Formally, the distribution $p_n(\x)$ at node $n$ is defined recursively as:
\begin{equation}
p_n(\x) \!:=\! 
\begin{cases}
g_n(\x) 
& \!\!\!\!  \text{$n$ is an input node,} \\[4pt]
\prod_{c \in \mathrm{ch}(n)} p_c(\x) 
& \!\!\!\!  \text{$n$ is a product node,} \\[6pt]
\sum_{c \in \mathrm{ch}(n)} \omega_{n,c}\, p_c(\x) 
& \!\!\!\! \text{$n$ is a sum node,}
\end{cases}
\label{eq:pc-recursive}
\end{equation}

where $\mathrm{ch}(n)$ denotes the children of node $n$, $g_n(\x)$ is a
univariate distribution (e.g., Categorical) over a variable
$X\in\X$, and $\omega_{n,c} \ge 0$ are learnable weights such that $\sum_{c \in \mathrm{ch}(n)} \omega_{n,c} = 1$.  Compared to LLMs, a defining advantage of PCs is that they support efficient exact and efficient inference for a large class of probabilistic queries, including arbitrary marginalization and conditional probability queries \citep{vergari2021compositional}. This tractability is enabled by structural constraints imposed on the circuit, most notably \emph{decomposability}, where the children of every product node have disjoint scopes, effectively simulating fully factorized distributions.

\section{A Unified Autoregressive Perspective}
\label{sec:unified-decomposition}

As established in the previous section, LLMs and PCs represent two distinct modeling frameworks: while LLMs rely on neural function approximators to learn autoregressive factors, PCs directly model the joint distribution by hierarchically composing simple distributions into complex ones. To facilitate a direct comparison, we observe that the tractable nature of PCs allows them to be equivalently cast in an autoregressive form, since PCs can compute arbitrary conditional probabilities efficiently, including the exact sequence of autoregressive factors defined in Equation~\ref{eq:autoregressive_factorization}. 

As a result, the PC’s inference procedure for computing these autoregressive factors induces a similar context-conditioned mapping as LLMs, \ie $\boldsymbol{\theta}_t = f_{\mathrm{PC}}(\x_{<t})$, which is analogous to $f_{\mathrm{LLM}}$ and also transforms a prefix into the predictive parameters of the distribution over $X_t$. In this sense, while a PC represents a joint distribution, it can be operationalized as an autoregressive computation graph. This allows us to directly compare the expressiveness of the Transformer’s forward pass $f$ with the PC’s conditional inference $h$, treating both as alternative mechanisms for generating next-token distributions.

With the unified perspective of $\boldsymbol{\theta}_t = \text{Encoder}(\x_{<t})$ established, where $f_{\mathrm{LLM}}$ and $f_{\mathrm{PC}}$ both serve as context-conditioned encoders, we can now investigate the sources of the expressiveness gap between both models.

In Transformer-based LLMs, the predictive parameters are typically produced by a linear output projection:
\begin{equation}
    \boldsymbol{\theta}_t^{\mathrm{LLM}}
    =
    \W^{\mathrm{LLM}}
    \mathbf{e}^{\mathrm{LLM}}(\x_{<t}),
    \qquad
    p^{\mathrm{LLM}}(X_t \mid \x_{<t})
    =
    \mathrm{Softmax}
    \big(
    \boldsymbol{\theta}_t^{\mathrm{LLM}}
    \big),
\label{eq:llm-output-layer}
\end{equation}
where $\mathbf{e}^{\mathrm{LLM}}(\x_{<t}) \in \mathbb{R}^{d}$ is the final hidden representation of the context $\x_{<t}$, and $\W^{\mathrm{LLM}} \in \mathbb{R}^{|\mathcal{V}| \times d}$. Since the output space is high-dimensional, with size $|\mathcal{V}|$, while the hidden dimension $d$ is typically much smaller, the attainable logit vectors are constrained by a low-dimensional projection. This dimensionality mismatch characterizes the output bottleneck in Transformer-based language models.

PCs face an analogous bottleneck during conditional inference. Unlike Transformer hidden units, which carry unconstrained continuous activations, PC nodes have explicit probabilistic semantics: their values correspond to probabilities. Furthermore, since sum nodes represent mixture distributions, each sum node can be associated with a categorical latent variable selecting among its children \citep{peharz2016latent}. When evaluating the conditional distribution of $X_t$, we identify a categorical latent variable $Z$ whose states correspond to the parent nodes of the input leaves for $X_t$. The resulting conditional is a mixture over latent states:
\begin{equation}
p^{\mathrm{PC}}(x_t \mid \x_{<t})
=
\sum_{z \in \mathrm{val}[Z]}
p^{\mathrm{PC}}(z \mid \x_{<t})\,
p^{\mathrm{PC}}(x_t \mid z).
\label{eq:pc-output-factor}
\end{equation}

This mixture can be written in matrix form. Let
$\mathbf{e}^{\mathrm{PC}}(\x_{<t}) \in \Delta^k$ denote the posterior vector over the $k=|\mathrm{val}[Z]|$ latent states, with entries
$\mathbf{e}^{\mathrm{PC}}_j(\x_{<t}) = p^{\mathrm{PC}}(z_j \mid \x_{<t})$.
Let $\W^{\mathrm{PC}} \in \mathbb{R}_{+}^{|\mathcal{V}| \times k}$ be the emission matrix, where
$[\W^{\mathrm{PC}}]_{i,j} = p^{\mathrm{PC}}(X_t=v_i \mid z_j)$.
Each column of $\W^{\mathrm{PC}}$ is a categorical distribution over the vocabulary. Therefore,
\begin{equation}
    \boldsymbol{\theta}_t^{\mathrm{PC}}
    :=
    p^{\mathrm{PC}}(X_t \mid \x_{<t})
    =
    \W^{\mathrm{PC}}
    \mathbf{e}^{\mathrm{PC}}(\x_{<t})
    \in \Delta^{|\mathcal{V}|}.
\label{eq:pc-output-matrix}
\end{equation}

This vectorized form makes the analogy with the Transformer output layer in Equation~\ref{eq:llm-output-layer} explicit. In both cases, a low-dimensional context representation is mapped to a vocabulary-sized predictive object. However, the two mappings operate in different spaces. In Transformers, $\boldsymbol{\theta}_t^{\mathrm{LLM}}$ is an unconstrained logit vector, which is subsequently normalized by a softmax. In PCs, $\boldsymbol{\theta}_t^{\mathrm{PC}}$ is already the probability vector itself: both the mixture coefficients $\mathbf{e}^{\mathrm{PC}}(\x_{<t})$ and the columns of $\W^{\mathrm{PC}}$ lie in probability simplices. Consequently, the PC prediction is constrained to lie in the convex hull of $k$ basis distributions. Thus, when $k \ll |\mathcal{V}|$, PCs exhibit an output bottleneck analogous to that of LLMs in dimensionality, but geometrically more restrictive because the bottleneck operates in probability space rather than logit space.

\paragraph{Context Encoding Bottleneck}
The second source of the gap is the capacity of the encoders $f$ and $h$ to capture dependencies within the prefix $\x_{<t}$. The self-attention mechanism of Transformers ($f_{\mathrm{LLM}}$) provides a highly flexible functional encoding. It allows every variable in the context to interact non-linearly, enabling the model to learn complex, long-range dependencies without structural overhead. By contrast, to support tractable and exact inference, PCs must adhere to structural constraints (e.g., decomposability). These constraints act as a rigid structural inductive bias, limiting the ``flow'' of information and the complexity of interactions that can be encoded in $f_{\mathrm{PC}}$. We will investigate how these structural biases fundamentally limit the PC's ability to represent the dense dependency structures that Transformers capture with ease.

\section{The Output Bottleneck}

As introduced in Section~\ref{sec:unified-decomposition}, PCs and LLMs share a common structural output bottleneck: a $d$-dimensional latent representation is projected into a $|V|$-dimensional space to parameterize the predictive distribution $\theta$. However, the key distinction lies in the geometric domain of this mapping. Transformer-based LLMs perform the projection in logit space, whereas PCs operate directly in probability space via convex combinations of basis distributions.

This difference suggests a potential representational disparity. Natural language next-token distributions are often low-entropy and can exhibit sharply peaked behavior given the context. In such cases, representing distributions as convex combinations in probability space may impose stronger geometric constraints than unconstrained linear mappings in logit space. To systematically investigate the impact of this ``output bottleneck'', we conduct controlled experiments across two model families: Hidden Markov Models (HMMs)~\citep{choi2020probabilistic}, as a representative class of PCs, and Transformers.

\noindent \textbf{Controlled Model Variants.}
An HMM can be viewed as a specific PC architecture where the latent variable $Z$ denotes the state, and the output mapping $W_{\text{PC}}$ corresponds to the emission matrix. In this setup, the embedding size $d$ is equivalent to the number of latent states. To isolate the effect of the output bottleneck, we introduce a Logit-HMM variant. Instead of the standard convex combination, the Logit-HMM extracts a feature vector $\mathbf{e}_t = \log p(z_t \mid \mathbf{x}_{\le t})$, representing the log-posterior over latent states. This feature is then passed through a linear layer $\mathbf{W} \in \mathbb{R}^{|V| \times d}$ to generate logits, which are normalized via $\mathrm{Softmax}$ to produce the final distribution.

To provide a parallel intervention within the Transformer family, we introduce a Prob-Transformer variant that replaces the standard logit-space output layer with a probability-space parameterization. Specifically, instead of producing logits followed by a Softmax, the model predicts a mixture over a set of basis distributions, where the mixture coefficients are directly derived from the hidden representation. This enforces a convex combination structure analogous to PCs, while preserving the same context encoder (see Appendix~\ref{app:prob-transformer}). We fix the Transformer architecture to 2 layers across all experiments.

We now evaluate these four model variants across datasets with different output scales to test whether the probability-space parameterization imposes a stronger output bottleneck than logit-space prediction. See Appendices~\ref{app:model-size} and~\ref{app:training} for parameter counts, architectural and training hyperparameters.

\begin{figure}[t]
    \centering
    \includegraphics[width=\linewidth]{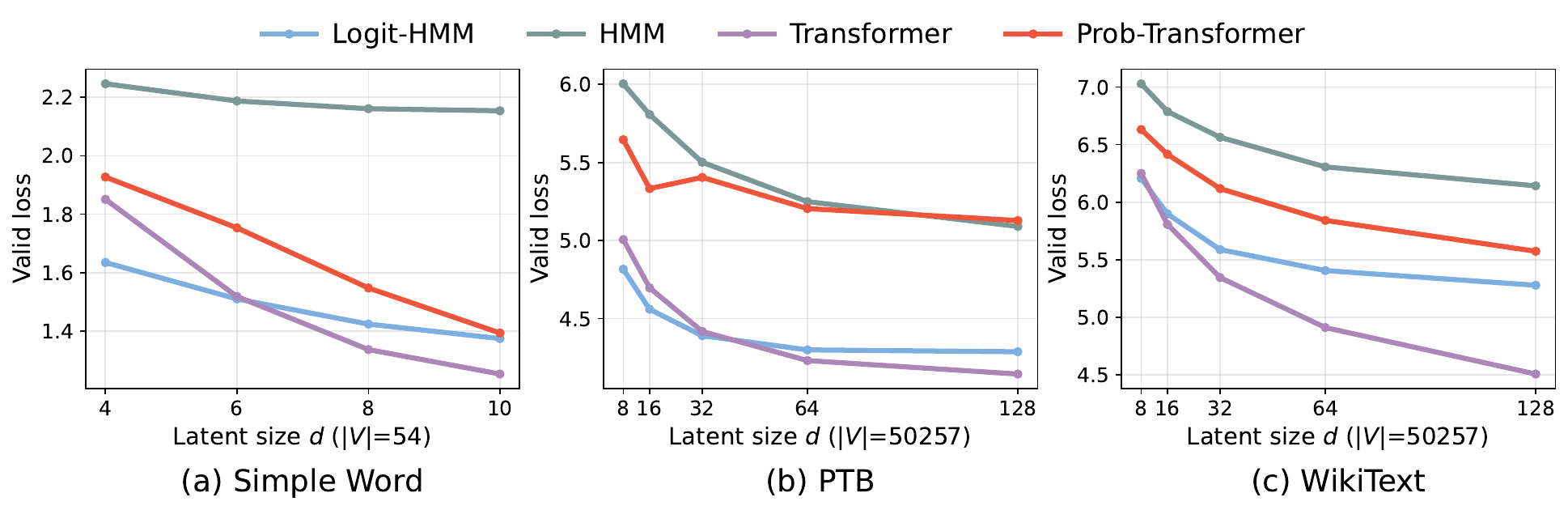}
    \caption{\textbf{Impact of output parameterization on modeling performance.} Lower validation loss is better. Across all datasets, logit-space variants improve over their probability-space counterparts within the same model family: Logit-HMM improves over HMM, and Transformer improves over Prob-Transformer. This supports the hypothesis that probability-space convex combinations impose a stronger geometric constraint than logit-space projections under a limited latent dimension.}
    \label{fig:output_bottleneck}
\end{figure}

\noindent \textbf{Controlled Sparse Prediction.}
To directly validate our hypothesis, we first use a simple word-level dataset of English words and roots \citep{lugosch2020pytorch} with a character-level tokenizer, where the task is to predict the final character from the prefix. This setting induces sparse, sharply peaked next-token distributions, since a prefix often nearly determines the final token. Because the vocabulary is small ($|\mathcal{V}|=54$), we sweep smaller latent sizes $d \in \{4,6,8,10\}$ so that $d \ll |\mathcal{V}|$ and the output bottleneck remains nontrivial. As shown in Figure~\ref{fig:output_bottleneck}~(a), the logit-space variants consistently outperform their probability-space counterparts under the same latent capacity. This indicates that even in a small-vocabulary setting, probability-space convex combinations are less efficient at representing sharp predictive distributions than logit-space projections.

\noindent \textbf{Real-World Language Modeling.}
To further validate our hypothesis on more realistic data, we extend the evaluation to two standard language modeling benchmarks: the Penn Treebank (PTB)~\citep{marcus1993building}, using $\sim$20K training sequences of length $64$, and a larger-scale WikiText-103~\citep{merity2016pointer} subset with 200K training sequences of length $128$. All models are trained using the standard autoregressive language modeling objective. We adopt the GPT-2 tokenizer~\citep{radford2019language} in both settings, yielding a vocabulary size of $|\mathcal{V}|=50{,}257$, and sweep $d \in \{8,16,32,64,128\}$.

In these settings, the output bottleneck becomes more pronounced since low-dimensional latent representations must be projected into a much larger vocabulary space. The results in Figure~\ref{fig:output_bottleneck}~(b) and (c) again show a consistent within-family pattern: logit-space variants outperform their probability-space counterparts across both datasets. On PTB, the Logit-HMM achieves performance comparable to the Transformer, suggesting that improving the output parameterization alone can substantially narrow the gap in this setting. On the larger WikiText subset, while the Logit-HMM remains substantially stronger than the standard HMM and also outperforms the Prob-Transformer, a gap with the Transformer persists. This indicates that once the probability-space output bottleneck is alleviated, additional factors such as context encoding become increasingly important at larger scale.

\section{The Context Encoding Bottleneck}
\label{sec:meta-token}

We now turn to the second source of the expressiveness gap: each encoder's capacity to fuse information across context variables. A model that fails to capture rich dependencies within $\x_{<t}$will produce poor conditionals. We therefore ask: how richly can each architecture combine information from distinct subsets of $\x_{<t}$?

In a Transformer, information fusion is carried out by self-attention. Attention weights are computed as a function of the input, allowing the model to dynamically rewire which positions interact and how strongly. This input-dependence lets the same architecture instantiate distinct effective interaction graphs depending on the context, capturing complex, long-range dependencies without prior structural commitment.
Unlike Transformers, PCs commit to a fixed interaction structure at design time. Cross-variable fusion occurs only at product nodes, whose children have pairwise disjoint scopes; sum nodes operate over children with identical scopes and thus do not directly combine information across different variables. Dependencies are instead captured indirectly: sum nodes assign different mixture weights to different factorization patterns, so variable interactions are modeled collectively through the global mixture structure.
To make this fixed structure precise, we focus on structured-decomposable PCs \citep{kisa2014probabilistic, dang2020strudel, loconte2024relationship}, where all scope splits are governed by a fixed binary tree called a \emph{vtree}.

\begin{defn}[Vtree]
\citep{pipatsrisawat2008new, kisa2014probabilistic}
A \emph{vtree} $\mathcal{T}$ over variables $\X$ is a full binary tree whose leaves are in bijection with $\X$. Each internal node $v \in \mathcal{T}$ is associated with a scope $S_v \subseteq \X$, defined recursively as $S_v = S_v^\ell \cup S_v^r$, where $S_v^\ell$ and $S_v^r$ are the scopes of its left and right children respectively, with $S_v^\ell \cap S_v^r = \emptyset$. The root node has scope $\X$.
\end{defn}

A PC is \emph{structured-decomposable} with respect to $\mathcal{T}$ if, at every product node $n$ with scope $S_n$, the scopes of its children correspond exactly to $(S_v^\ell, S_v^r)$ for the unique vtree node $v$ such that $S_v = S_n$.
This means that the vtree fully enumerates the set of all possible variable interactions, \emph{making the expressiveness constraints imposed by the fixed structure explicit and analyzable.}


To quantify and compare the richness of variable fusing in PCs and LLMs, we need a formal measure of \emph{how much a function can couple information across a partition of its inputs}. We adopt \emph{separation rank} as our primary complexity measure.

\begin{defn}[Separation Rank]
Let $\X$ be a set of discrete variables, and let $(A,B)$ be a partition of $\X$ (so $A\cup B=\X$ and $A\cap B=\emptyset$).
For a scalar function $f(\x_A,\x_B)$, the separation rank of $f$ with respect to $(A,B)$ is defined as
\begin{equation*}
\begin{aligned}
\operatorname{sep}&(f;A,B)
:=\min\Big\{R\in\mathbb{N}_{+}:\ \exists \{g_r\}_{r=1}^R,\{h_r\}_{r=1}^R, \text{such that } f(\x_A,\x_B)=\sum_{r=1}^R g_r(\x_A)\,h_r(\x_B)\Big\}.
\end{aligned}
\end{equation*}
\end{defn}

Intuitively, $\mathrm{sep}(f; A, B)$ is the minimum number of rank-one cross terms needed to express $f$ across the cut $(A, B)$. A separation rank of $1$ corresponds to a fully factored function $f(\x_A, \x_B) = g(\x_A) h(\x_B)$, meaning the two sides of the partition contribute independently. A high separation rank, on the other hand, indicates that the function encodes many linearly independent interaction patterns between $A$ and $B$, and hence cannot be decomposed without a large number of terms.




\paragraph{Existence of High-Rank Partitions.} Prior work on Transformer expressiveness has established that a Transformer of embedding dimension $d$ and depth $L<\log_3 d$ achieves the following separation rank \citep{yang2024etas}:
$
    \log \operatorname{sep}(f_{\text{LLM}}; A, B) = \widetilde{O}(d \log d)
$.
A naive expectation is that the fixed routing structure of structured-decomposable PCs fundamentally limits their separation rank, resulting in much weaker variable interactions compared to Transformers across all partitions.

Surprisingly, we show that for any structured-decomposable PC, there always exists a partition of the variables along which the circuit achieves separation rank comparable to that of a Transformer.

\begin{thm}
\label{thm:pc-high-rank}
For any structured-decomposable PC with vtree $\mathcal{T}$, there exists a partition 
$(A^\star, B^\star)$ of the variables that is aligned with $\mathcal{T}$ such that 
$
    \log \operatorname{sep}(f_{\mathrm{PC}}; A^\star, B^\star)    = \widetilde{O}(T \log d)
$,
where $T$ is the number of variables and $d$ is the maximum interface dimension of the circuit, which informally measures the expressive capacity of individual nodes. Formal definitions are provided in Appendix~\ref{app:proofs}.
\end{thm}


Since the interface dimension $d$ in a PC plays an analogous role to the embedding dimension in a Transformer, and $T$ is typically on the same order as $d$ in the LLM setting, the two bounds are asymptotically comparable.
Therefore, the two bounds are of the same asymptotic order. 

\section{The Cost of Fixed Routing}

Theorem~\ref{thm:pc-high-rank} shows that PCs can match Transformers in maximum separation rank. However, this capacity is realized only when the vtree aligns with the data's dependency topology. 

We now demonstrate that this alignment requirement is a fundamental limitation: PCs suffer severe performance degradation under structural mismatch (Sec.~\ref{sec:structural-selectivity}), and real-world data, where different samples exhibit different dependency topologies, inherently demands the sample-dependent routing that only Transformers provide (Sec.~\ref{subsec:transformer-mixed-topology}).




\subsection{Structural Selectivity in Probabilistic Circuits}
\label{sec:structural-selectivity}


We construct a controlled setting where the dependency topology of the target distribution can be explicitly manipulated, which allows us to directly test PC performance under matched and mismatched vtree layouts against a Transformer baseline.

\paragraph{Setup.}
We use two synthetic autoregressive datasets, each emphasizing a different dependency topology (details in Appendix~\ref{app:synthetic_data}). For PCs, we use a structured-decomposable model with a balanced binary vtree, fixing the architecture and varying only the \emph{leaf order}. We compare a \emph{standard} layout (original token order, aligned with local dependencies) and a \emph{shifted-induction} layout (pairing each prefix position with its copy target, aligned with induction-style dependencies). Each layout is matched to one dataset and mismatched to the other. We compare against a standard causal Transformer trained on the same next-token objective; all models are evaluated by token-level NLL on held-out data. Full experimental details are in Appendix~\ref{appx:exp-details}.

\begin{table}[t]
\caption{Performance of structured-decomposable PCs with two different vtree 
    layouts and a Transformer baseline on the local-copy and induction-style 
    prefix-copy datasets. The standard layout is aligned with local dependencies 
    and the shifted-induction layout is aligned with nonlocal prefix-copy 
    dependencies. A matched layout achieves performance close to the Transformer, 
    while a mismatched layout degrades substantially, consistent with 
    Theorem~\ref{thm:pc-high-rank} and Proposition~\ref{prop:strict-non-attainment}.}
    \centering
\begin{tabular}{llll}
\toprule
              & \begin{tabular}[c]{@{}l@{}}\tt Shifted-induction\\ \tt PC Model\end{tabular} & \tt Standard Model & \tt Transformer \\ \midrule
Adjacent Data & 6.283$\pm$0.002                                                                & 3.158$\pm$0.002          & \textbf{3.018$\pm$0.000}       \\
Distant Data  & 3.144$\pm$0.000                                                                & 6.297$\pm$0.001          & \textbf{3.018$\pm$0.000}      \\
\bottomrule
\end{tabular}
\label{fig:exp1-layout-comparison}
\end{table}

\paragraph{Result.} Table~\ref{fig:exp1-layout-comparison} shows that when the vtree layout is aligned with the data's dependency topology, the PC achieves NLL close to the Transformer. This is consistent with Theorem~\ref{thm:pc-high-rank}: the aligned vtree places the relevant variable interactions along high-rank cuts, allowing the PC to capture them as effectively as a Transformer of comparable size. However, this high-rank capability is not universal.
Formally, the high separation rank from Theorem~\ref{thm:pc-high-rank} is attainable only on a restricted family of vtree-aligned partitions; any degeneracy in the circuit's local computations yields strictly smaller rank (Proposition~\ref{prop:strict-non-attainment} in the appendix). The mismatched layout in Table~\ref{fig:exp1-layout-comparison} confirms this: NLL increases substantially, indicating that the relevant partition falls outside the high-rank family.

\subsection{The Necessity of Topology Alignment}
\label{subsec:transformer-mixed-topology}

The previous experiment shows that a single vtree cannot simultaneously align with multiple dependency topologies. In practice, natural data rarely exhibits a uniform dependency pattern. For instance, language modeling requires both local phrase-level interactions and long-range structural dependencies such as subject-verb agreement across clauses. We further investigate a broader question: is this limitation specific to the PC framework, or does it reflect a more general cost of fixing the routing structure regardless of model family?

To isolate the effect of structural commitment, we compare three Transformers trained on a mixed-topology dataset $\mathcal{D}_{\mathrm{mix}}$ combining local-copy and induction-style tasks: (i) a standard causal Transformer with unconstrained attention, (ii) an \emph{adjacent}-layout variant whose attention masks follow a fixed tree aligned with local dependencies, and (iii) a \emph{distant}-layout variant aligned with induction-style dependencies. Both fixed-routing variants use $\log_2 T$ layers matching the tree depth. We evaluate on both the synthetic mixed-topology dataset $\mathcal{D}_{\mathrm{mix}}$ and PTB to test whether the cost of fixed routing persists on real data.


As shown in Table~\ref{fig:exp2-transformer-routing}, the vanilla Transformer achieves substantially lower NLL than both fixed-routing variants. This confirms that the cost of structural commitment is not specific to PCs: no single fixed topology can simultaneously achieve high separation rank for both local and nonlocal dependencies, whereas input-dependent attention dynamically routes information according to each sample's dependency structure.

\begin{table}[t]
\centering
\renewcommand{\arraystretch}{1.3}
\caption{Validation NLL of Transformers with different attention routing structures trained on the mixed-topology dataset. Vanilla denotes a standard causal Transformer with unconstrained self-attention. Adj Attn and Dist Attn denote Transformers with fixed tree-structured attention masks aligned with local and induction-style dependencies respectively. The vanilla Transformer maintains strong performance across both validation sets, while fixed-routing variants degrade on the topology mismatched with their attention layout, demonstrating that structural commitment to a fixed routing pattern incurs a performance cost under mixed dependency topologies regardless of model family.}
\label{fig:exp2-transformer-routing}
\begin{tabular}{lccc}
\toprule
Mask      & \tt vanilla         & \tt adjacent        & \tt distant         \\ \midrule
synthetic & \textbf{3.077$\pm$0.057} & 3.312$\pm$0.461 & 3.154$\pm$0.016 \\
PTB       & \textbf{3.818$\pm$0.008} & 3.827$\pm$0.007 & 3.871$\pm$0.006 \\
\bottomrule
\end{tabular}
\end{table}

\section{Relaxing Structured Decomposability}
\label{subsec:mixed-structure-pc}
The preceding experiments show that fixed routing is a fundamental bottleneck. In structured-decomposable PCs, this rigidity originates from the vtree, which fixes all scope splits at design time. A natural question is whether we can improve routing flexibility by relaxing structured decomposability to the broader class of decomposable PCs, while retaining the tractability guarantees of the PC framework. We first provide theoretical justification: decomposable PCs are strictly more expressive than structured-decomposable ones.


\begin{thm}[Limitation of structured-decomposable PCs, informal]
\label{thm:selector-separation-informal}
There exists a function $C$ on $n$ variables that admits a polynomial-size decomposable PC, but does not admit any polynomial-size structured-decomposable PC. See Theorem~\ref{thm:selector-separation} for a formal statement.
\end{thm}

This inspires us to investigate empirically how much this relaxation matters in practice. We evaluate two regimes: a controlled synthetic setting and two real language modelling benchmarks (PTB~\citep{marcus1993building} and UD-ATIS~\citep{ud_english_atis}. Full details are in Appendix~\ref{app:exp-details-real}.

\paragraph{Synthetic Setting.}
We combine the standard and shifted-induction PCs from the previous experiments via a shared sum node at the root. The resulting model is decomposable but not structured-decomposable, as its branches impose conflicting scope-splitting patterns. 
\paragraph{Real Datasets.}
For PTB and UD-ATIS, we adopt a \emph{specialist initialization} strategy: we independently train structured-decomposable components, each under a distinct vtree layout and trained on the subset of data whose dependency topology best matches that layout. We compare this against a \emph{scratch mixture} trained end-to-end with matched parameter budget. 


\paragraph{Results.}
On the synthetic $\mathcal{D}_{\mathrm{mix}}$, the fused decomposable PC achieves the lowest NLL among all PC variants, validating Theorem~\ref{thm:selector-separation-informal}: relaxing structured decomposability allows the model to combine complementary routing structures, recovering expressiveness that no single vtree can provide.
However, the gains diminish as task complexity increases (Table~\ref{tab:specialist-init}). On UD-ATIS, the mixture still improves over the best single-vtree PC, but the margin is smaller than on the synthetic task. On PTB, the initialized mixture barely improves over the single-vtree baseline. This points to a key open challenge: while the theoretical expressiveness gap between decomposable and structured-decomposable PCs is clear, current methods for constructing and optimizing decomposable PCs do not yet scale effectively enough to close the gap with Transformers on real-world distributions.

\begin{table}[t]
  \centering
  \caption{Comparison of mixture training strategies (test NLL; lower is better). \emph{Best Single}: the best-performing single structured-decomposable PC selected by validation. \emph{Init Mix}: each component is a structured-decomposable PC pre-trained under a distinct vtree layout, combined via a sum node. \emph{Scratch Mix}: a decomposable mixture with the same architecture trained end-to-end. Bold indicates the better mixture variant.}
  \label{tab:specialist-init}
  \begin{tabular}{lccc}
    \toprule
    Dataset & \tt Best Single & \tt Init Mix & \tt Scratch Mix \\
    \midrule
    $\mathcal{D}_{\mathrm{mix}}$ (synthetic) & 6.219$\pm$0.005 & \textbf{5.431$\pm$0.018} & 6.117$\pm$0.030 \\
    UD-ATIS                                   & 6.975$\pm$0.035 & 6.779$\pm$0.012 & \textbf{6.755$\pm$0.031} \\
    PTB                                       & \textbf{6.527$\pm$0.000} & \textbf{6.527$\pm$0.000} & 6.531$\pm$0.000 \\
    \bottomrule
  \end{tabular}
\end{table}

\section{Related Work}

\paragraph{Expressivity of Probabilistic Circuits.}
Extensive research has analyzed the theoretical capacity of PCs. Martens and Medabalimi~\citep{martens2014expressive} proved fundamental limitations of decomposable and complete SPNs by showing that certain tractable distributions cannot be represented by such SPNs without exponential size. De Colnet and Mengel~\citep{de2021compilation} provided a systematic succinctness map for arithmetic circuits under combinations of structural restrictions including decomposability, determinism, and structured decomposability, proving unconditional exponential separations between several circuit classes. In a complementary direction, it has been shown that, in the context of SPNs, consistency does not yield an exponential succinctness advantage over decomposability, as any complete and consistent SPN can be converted into a decomposable one with only a polynomial increase in size~\citep{peharz2015theoretical}.
Other theoretical perspectives include sample complexity limits~\citep{aden2020sample}, relationships between circuit classes~\citep{wang2025relationship}, and the hardness of distribution approximation~\citep{leland2025hardness}. Particularly relevant to our analysis is the formal connection between PCs and structured tensor factorizations~\citep{loconte2024relationship}.





\paragraph{Theoretical Analysis of Transformers.}

Prior work has extensively analyzed the expressivity of Large Language Models, identifying a ``softmax bottleneck'' where the low-dimensional latent space restricts the rank of the predicted output distribution \citep{wies2021transformer,yang2017breaking}. However, recent work has shown that this bottleneck may not significantly limit the probabilities assigned to the most likely tokens~\citep{basri2026softmax}. Beyond expressivity, Godey and Artzi~\citep{godey2026lost} further showed that this bottleneck also impedes optimization, as the rank-deficient output projection suppresses the vast majority of the gradient signal during backpropagation. Furthermore, theoretical studies have quantified the context-encoding capacity of self-attention using separation rank, demonstrating that Transformers can dynamically maintain high-rank interactions across arbitrary input partitions \citep{levine2021inductive,yang2024etas}.Our work bridges these analyses with tractable modeling, directly contrasting the Transformer's low-rank logit-space outputs and dynamic routing with the probability-space mixtures of PCs.



\paragraph{Empirical Advancements in Probabilistic Circuits.}
Recent work has significantly scaled Probabilistic Circuits to complex, high-dimensional distributions. A major line of this research focuses on structural innovations \citep{adel2015learning, liu2021tractable, loconte2024relationship, zhang2025scaling,liu2022scaling,liu2023understanding}. However, these architectures predominantly rely on strict structured decomposability to guarantee efficient inference. Beyond structural design, the empirical success of modern PCs is further driven by continuous advancements in parameter learning \citep{poon2011sum, peharz2020einsum, liu2025rethinking} and systems frameworks \citep{molina2019spflow, liu2024scaling}.

\section{Conclusion and Limitations}\label{sec:conclusion}

In this work, we established a unified autoregressive framework to directly compare the expressive power of PCs and LLMs. Our analysis reveals two structural bottlenecks underlying the expressivity gap. First, an output bottleneck: PCs parameterize predictions as convex combinations in probability space, which struggles to represent sharp distributions; adopting a logit-space parameterization substantially narrows this gap. Second, a context encoding bottleneck: we proved that structured-decomposable PCs can match Transformer separation rank, but only on vtree-aligned partitions, and that this limitation is fundamental to structured decomposability, as decomposable PCs are provably more expressive.
These findings suggest two concrete directions for improving tractable probabilistic models: exploring output parameterizations beyond probability-space convex combinations, and developing learning algorithms that can effectively exploit the broader class of decomposable PCs. However, our analysis has limitations. Our separation rank results characterize worst-case expressivity gaps, and the context encoding bottleneck is primarily validated on synthetic datasets with controlled dependency topologies. Moreover, while decomposable PCs are provably more expressive than structured-decomposable ones, realizing this advantage in practice remains an open challenge that we leave for future work.

\bibliographystyle{plain}
\bibliography{ref}

\newpage
\appendix
\section{Proofs}
\label{app:proofs}

\subsection{Preliminary Definitions}

For a structured-decomposable PC with vtree $\mathcal{T}$ over variables 
$X = \{X_1, \dots, X_T\}$, we introduce the following definitions used throughout 
the proofs.

\paragraph{Local state dimension.} For each vtree node $v$, let $\mathcal{H}_v$ 
denote the space of functions exported from the subtree rooted at $v$ toward the 
rest of the circuit. We define the \emph{local state dimension}
\[
    d_v := \dim(\mathcal{H}_v).
\]
In a normalized, non-redundant implementation, $d_v$ coincides with the interface 
width $|I_v|$; in general, $d_v \leq |I_v|$.

\paragraph{Frontier sets.} Given a partition $X = A \sqcup B$, a vtree node $v$ 
is called \emph{$A$-pure} if $U_v \subseteq A$, \emph{$B$-pure} if 
$U_v \subseteq B$, and \emph{mixed} otherwise. We define the maximal pure frontier 
sets as:
\begin{itemize}
    \item $\mathcal{C}_A$: the set of all $A$-pure nodes whose parent is mixed;
    \item $\mathcal{C}_B$: the set of all $B$-pure nodes whose parent is mixed.
\end{itemize}

\paragraph{Inductive rank quantity.} Define $R_A(v)$ as the minimal integer $R$ 
such that there exist functions $\phi_1^v, \dots, \phi_R^v : \mathcal{X}_{U_v 
\cap A} \to \mathbb{R}$ with the property that for every $f \in \mathcal{H}_v$, 
there exist coefficient functions $\alpha_1^f, \dots, \alpha_R^f : 
\mathcal{X}_{U_v \cap B} \to \mathbb{R}$ satisfying
\[
    f(x_{U_v}) = \sum_{i=1}^{R} \phi_i^v(x_{U_v \cap A})\, \alpha_i^f(x_{U_v \cap B}).
\]
Intuitively, $R_A(v)$ is the minimum number of basis functions on the $A$-side 
needed to represent all exported functions from subtree $v$, with $B$-dependent 
coefficients.

\subsection{Proof of Theorem \ref{thm:pc-high-rank}}

\begin{lemma}\label{lem:subtree-basis}
For any vtree node $v$ with children $L$ and $R$:
\begin{enumerate}
    \item If $v$ is $A$-pure, then $R_A(v) \leq d_v$.
    \item If $v$ is $B$-pure, then $R_A(v) = 1$.
    \item If $v$ is mixed, then $R_A(v) \leq R_A(L) \cdot R_A(R)$.
\end{enumerate}
\end{lemma}

\begin{proof}
\textbf{Case 1} ($v$ is $A$-pure). Since $U_v \cap B = \emptyset$, all functions 
in $\mathcal{H}_v$ depend only on $x_{U_v \cap A} = x_{U_v}$. Let 
$\{\phi_i^v\}_{i=1}^{d_v}$ be any basis of $\mathcal{H}_v$. Then every 
$f \in \mathcal{H}_v$ can be written as
\[
    f(x_{U_v}) = \sum_{i=1}^{d_v} c_i^f\, \phi_i^v(x_{U_v}),
\]
where $c_i^f \in \mathbb{R}$ are constants, which we view as trivial functions on 
the empty $B$-side. Hence $R_A(v) \leq d_v$.

\textbf{Case 2} ($v$ is $B$-pure). Since $U_v \cap A = \emptyset$, every 
$f \in \mathcal{H}_v$ does not depend on $x_A$ at all. Taking $\phi_1^v \equiv 1$ 
as the single basis function and $\alpha_1^f = f(x_{U_v})$ as the coefficient, 
we obtain $R_A(v) = 1$.

\textbf{Case 3} ($v$ is mixed). Let $r_L = R_A(L)$ and $r_R = R_A(R)$. By 
definition, there exist basis functions $\{\phi_i^L\}_{i=1}^{r_L}$ on 
$U_L \cap A$ and $\{\phi_j^R\}_{j=1}^{r_R}$ on $U_R \cap A$ such that for any 
$f_L \in \mathcal{H}_L$ and $f_R \in \mathcal{H}_R$:
\[
    f_L = \sum_{i=1}^{r_L} \phi_i^L\, \alpha_i^{f_L}, \qquad
    f_R = \sum_{j=1}^{r_R} \phi_j^R\, \beta_j^{f_R}.
\]
Since the PC is structured-decomposable, any product gate at $v$ combines a 
function $f_L \in \mathcal{H}_L$ with a function $f_R \in \mathcal{H}_R$ over 
disjoint scopes. Their product satisfies:
\begin{align*}
    f_L \cdot f_R 
    =& \left(\sum_i \phi_i^L \alpha_i^{f_L}\right)
      \left(\sum_j \phi_j^R \beta_j^{f_R}\right)\\
    =& \sum_{i=1}^{r_L} \sum_{j=1}^{r_R}
      \underbrace{(\phi_i^L \phi_j^R)}_{\text{depends on } U_v \cap A}
      \underbrace{(\alpha_i^{f_L} \beta_j^{f_R})}_{\text{depends on } U_v \cap B}.
\end{align*}
The tensor-product functions $\{\phi_i^L \phi_j^R\}$ span the $A$-side of any 
product term, yielding at most $r_L r_R$ basis functions. Sum gates take linear 
combinations of functions with the same scope and do not increase the span 
dimension. Therefore all functions in $\mathcal{H}_v$ are spanned by these 
$r_L r_R$ tensor-product basis functions, giving
\[
    R_A(v) \leq R_A(L) \cdot R_A(R). \qedhere
\]
\end{proof}


\begin{lemma}\label{lem:general-cut-sep-rank}
Let $y$ be the output function of a structured-decomposable PC over variables $X$, 
with vtree $\mathcal{T}$. For any partition $X = A \sqcup B$,
\[
    \operatorname{sep}(y;\, A, B)
    \;\leq\;
    \min\!\left\{
        \prod_{v \in \mathcal{C}_A} d_v,\;
        \prod_{u \in \mathcal{C}_B} d_u
    \right\}.
\]
\end{lemma}

\begin{proof}
Since $y \in \mathcal{H}_r$ (the output function belongs to the root's exported 
function space), the definition of $R_A(r)$ directly gives a separation 
representation of $y$ with $R_A(r)$ terms, hence
\[
    \operatorname{sep}(y;\, A, B) \leq R_A(r).
\]
Applying Lemma~\ref{lem:subtree-basis} recursively from the root down to the 
frontier $\mathcal{C}_A \cup \mathcal{C}_B$, at each mixed node the bound 
multiplies, and at each pure node it terminates. This gives:
\begin{align*}
    R_A(r)
    \;\leq\;&
    \prod_{v \in \mathcal{C}_A} R_A(v) \cdot \prod_{u \in \mathcal{C}_B} R_A(u)\\
    \;\leq\;&
    \prod_{v \in \mathcal{C}_A} d_v \cdot \prod_{u \in \mathcal{C}_B} 1
    \;=\;
    \prod_{v \in \mathcal{C}_A} d_v.
\end{align*}
By the symmetric argument (defining $R_B(v)$ analogously and swapping the roles 
of $A$ and $B$), we also obtain
\[
    \operatorname{sep}(y;\, A, B) \leq \prod_{u \in \mathcal{C}_B} d_u.
\]
Taking the minimum of the two bounds completes the proof.
\end{proof}


\begin{lemma}\label{thm:exact-sep-rank}
For any non-root vtree node $v$, let $A = U_v$ and $B = X \setminus U_v$. Let 
$\{\phi_1, \dots, \phi_{d_v}\}$ be a basis for $\mathcal{H}_v$, and write
\[
    y(x_A, x_B) = \sum_{i=1}^{d_v} \phi_i(x_A)\, \alpha_i(x_B).
\]
Define the \emph{effective transfer rank} $r_v := \dim \operatorname{span}
\{\alpha_1, \dots, \alpha_{d_v}\}$. Then
\[
    \operatorname{sep}(y;\, A, B) = r_v \;\leq\; d_v.
\]
\end{lemma}

\begin{proof}
\textbf{Upper bound} ($\operatorname{sep} \leq r_v$). Let $\{\beta_1, \dots, 
\beta_{r_v}\}$ be a basis for $\operatorname{span}\{\alpha_i\}$, so that 
$\alpha_i = \sum_{j=1}^{r_v} c_{ij} \beta_j$. Substituting:
\[
    y(x_A, x_B)
    = \sum_{j=1}^{r_v} \underbrace{\left(\sum_{i=1}^{d_v} c_{ij}\,
      \phi_i(x_A)\right)}_{\text{depends only on }A}
      \beta_j(x_B),
\]
which is a decomposition into $r_v$ rank-one terms. Hence 
$\operatorname{sep}(y; A, B) \leq r_v$.

\textbf{Lower bound} ($\operatorname{sep} \geq r_v$). Suppose for contradiction 
that $\operatorname{sep}(y; A, B) = m < r_v$, so that
\[
    y(x_A, x_B) = \sum_{k=1}^{m} g_k(x_A)\, h_k(x_B).
\]
Since $\{\phi_i\}$ is a basis for $\mathcal{H}_v$, each $g_k$ can be written as 
$g_k = \sum_{i=1}^{d_v} a_{ik} \phi_i$. Substituting and comparing coefficients 
of each $\phi_i$ (which are linearly independent):
\[
    \alpha_i(x_B) = \sum_{k=1}^{m} a_{ik}\, h_k(x_B) \quad \forall\, i.
\]
This implies $\operatorname{span}\{\alpha_i\} \subseteq \operatorname{span}
\{h_1, \dots, h_m\}$, so $r_v \leq m$, a contradiction. Therefore 
$\operatorname{sep}(y; A, B) \geq r_v$.

Combining both bounds gives $\operatorname{sep}(y; A, B) = r_v$, and 
$r_v \leq d_v$ holds by construction.
\end{proof}

\begin{corollary}[Exact attainment of the upper bound]
\label{cor:exact-attainment}
If the transfer at node $v$ is full-rank, i.e., $r_v = d_v$, then
\[
    \operatorname{sep}(y;\, U_v,\, X \setminus U_v) = d_v,
\]
and the frontier-product upper bound is exactly attained on the vtree-induced 
cut $(U_v,\, X \setminus U_v)$.
\end{corollary}

Combine above results, we have the following result.
For any structured-decomposable PC with vtree $\mathcal{T}$, there exists a partition 
$(A^\star, B^\star)$ of the variables that is aligned with $\mathcal{T}$ such that 
\begin{align*}
    \log \operatorname{sep}(y; A^\star, B^\star) \leq& \min\{\sum_{v \in \mathcal{C}_{A^\star}} \log d_v, \sum_{v \in \mathcal{C}_{B^\star}} \log d_v\}\\
    =& \widetilde{O}(T \log d),
\end{align*}

\subsection{Proof of Strict Non-Attainment of the Frontier Bound}
\begin{prop}[Strict non-attainment of the frontier bound]
\label{prop:strict-non-attainment}
Let $U_A := \prod_{v \in \mathcal{C}_A} d_v$. 
If at least one of the following 
conditions holds, then $\operatorname{sep}(y; A, B) < U_A$:
\begin{enumerate}
    \item \textbf{Frontier under-utilization:} there exists $v \in \mathcal{C}_A$ 
    such that $R_A(v) < d_v$;
    \item \textbf{Non-full tensor growth at a mixed node:} there exists a mixed 
    node $t$ with children $L, R$ such that $R_A(t) < R_A(L) \cdot R_A(R)$;
    \item \textbf{Root output does not span the full root basis:} 
    $\operatorname{sep}(y; A, B) < R_A(r)$.
\end{enumerate}
\end{prop}
\begin{proof}
From the recursive bound established in the proof of 
Theorem~\ref{lem:general-cut-sep-rank}, we have the chain of inequalities:
\begin{align*}
    \operatorname{sep}(y;\, A, B)
    \;\leq\; R_A(r)
    \;\leq\;& \prod_{v \in \mathcal{C}_A} R_A(v) \cdot \prod_{u \in \mathcal{C}_B} R_A(u)\\
    \;\leq\;& \prod_{v \in \mathcal{C}_A} d_v
    \;=\; U_A,
\end{align*}
where the last factor uses $R_A(u) = 1$ for all $u \in \mathcal{C}_B$. If any of 
the three conditions holds, the corresponding inequality in this chain becomes 
strict, and the strict inequality propagates through the remaining steps, yielding 
$\operatorname{sep}(y; A, B) < U_A$.
\end{proof}

Proposition~\ref{prop:strict-non-attainment} complements 
Corollary~\ref{cor:exact-attainment}: while the corollary identifies when the 
bound is exactly attained (full-rank transfer at a vtree-induced cut), the 
proposition identifies three structural sources of slack. In practice, conditions 
(1) and (2) arise naturally when parameters are degenerate or when the circuit has 
redundant components; condition (3) reflects the additional gap between the root's 
internal representation and the scalar output $y$.

\subsection{Proof of Theorem~\ref{thm:selector-separation-informal}}
\begin{prop}[No universal polynomial restructuring]
\label{prop:no-universal-restructuring}
Assume $\mathrm{FP} \neq \#\mathrm{P}$.
Then there is no polynomial-time algorithm $\mathcal{R}$ with the following property:

for every polynomial-size structured-decomposable and deterministic PC $P$ and every target
vtree $T$, the algorithm $\mathcal{R}(P,T)$ outputs an equivalent polynomial-size
structured-decomposable and deterministic PC respecting $T$.
\end{prop}

\begin{proof}
Suppose, for contradiction, that such an algorithm $\mathcal{R}$ exists.

Let $P$ and $Q$ be two polynomial-size structured-decomposable and deterministic PCs over
the same variable set $X$, possibly respecting different vtrees.
Let $T_Q$ be the vtree respected by $Q$.
Apply $\mathcal{R}$ to restructure $P$ into an equivalent structured-decomposable and
deterministic PC
\[
P' := \mathcal{R}(P,T_Q)
\]
that respects the same vtree $T_Q$.

Therefore, by Theorem B.2~\citep{vergari2021compositional}, their product
$P' \cdot Q$ can be computed in time polynomial in $|P'|$ and $|Q|$.
Since $P'$ is equivalent to $P$, this computes the product $P \cdot Q$ in polynomial time.

However, \cite{vergari2021compositional} show that computing the product of two
structured-decomposable and deterministic circuits as a decomposable circuit is
$\#\mathrm{P}$-hard.
Thus we would obtain a polynomial-time algorithm for a $\#\mathrm{P}$-hard problem, implying
$\mathrm{FP}=\#\mathrm{P}$, a contradiction.

Hence no such universal polynomial-time restructuring algorithm exists.
\end{proof}
\begin{theorem}[Limitation of structured-decomposable PCs]
\label{thm:selector-separation}
Let $\Y$ be a set of variables, and let $s$ be an additional binary variable.
Let
\[
A,B : \mathcal{X}_{\Y} \to \mathbb{R}_{\ge 0}
\]
be two functions such that:
\begin{enumerate}
    \item $A$ admits a polynomial-size structured-decomposable PC;
    \item $B$ admits a polynomial-size structured-decomposable PC;
    \item there is no vtree $T$ on $\Y$ such that both $A$ and $B$ admit polynomial-size structured-decomposable PCs respecting $T$.
\end{enumerate}
Define
\[
F(\y,s)
=
\mathbf{1}[s=1]\,A(\y)
+
\mathbf{1}[s=0]\,B(\y).
\]
Then $F$ admits a polynomial-size decomposable PC, but does not admit a polynomial-size structured-decomposable PC.
\end{theorem}
\begin{proof}
We first show that $F$ admits a polynomial-size decomposable PC.
Construct a circuit with two product nodes,
\[
G_1 = \mathbf{1}[s=1]\cdot A(\y),
\qquad
G_2 = \mathbf{1}[s=0]\cdot B(\y),
\]
and a sum node at the root computing
\[
F(\y,s)=G_1+G_2.
\]
Each product node is decomposable, since one child depends only on $s$ and the other depends only on $\Y$, so their scopes are disjoint. Hence the whole circuit is decomposable. Its size is polynomial.

Now suppose, for contradiction, that $F$ admits a polynomial-size structured-decomposable PC $C$ respecting some vtree $\mathcal{T}$ over $\Y \cup \{s\}$.

Condition on the selector variable. Setting $s=1$ yields
\[
F(\y,1)=A(\y),
\]
and setting $s=0$ yields
\[
F(\y,0)=B(\y).
\]
Therefore, the conditioned circuits
\[
C\mid_{s=1}
\qquad\text{and}\qquad
C\mid_{s=0}
\]
compute $A$ and $B$, respectively.

Moreover, conditioning preserves structured decomposability and does not increase circuit size. After restricting $s$, both conditioned circuits respect the same restricted vtree $\mathcal{T}|_{\Y}$ on $\Y$.

Hence both $A$ and $B$ admit polynomial-size structured-decomposable PCs respecting one common vtree $\mathcal{T}|_{\Y}$, contradicting assumption (3).

Therefore $F$ does not admit a polynomial-size structured-decomposable PC.
\end{proof}

\section{Experimental Details}
\label{appx:exp-details}

\subsection{Model Architectures}

\paragraph{Balanced-tree PC.}
All PC models are structured-decomposable with respect to a full balanced binary 
tree vtree over the $N$ sequence positions, where $N$ is required to be a power 
of two. The tree is constructed by repeatedly halving the sequence: at each 
level, adjacent pairs of nodes are merged by alternating sum and product layers, 
until a single root node remains. Concretely, the architecture consists of 
$\log_2 N$ stages, each comprising a \texttt{SumLayer} followed by a 
\texttt{ProductLayer}, and a final \texttt{SumLayer} at the root with a learned 
mixture weight vector. Each node maintains $C$ channels throughout. The leaf 
nodes are parameterized as categorical distributions over the vocabulary, 
implemented as a \texttt{CategoricalLeaf} module that outputs per-position, 
per-channel log-probabilities.

\paragraph{Tree layouts.}
The vtree structure is controlled by a \emph{leaf permutation}, which maps tree 
positions to original sequence indices. We use two layouts:

\textbf{Standard layout.} The identity permutation
$$
    \pi_{\mathrm{std}}(i) = i, \quad i = 0, \dots, N-1,
$$
so that tree position $i$ corresponds to original token $i$. Under this layout, 
adjacent tokens are placed as siblings at the lowest level of the tree, making 
the model naturally aligned with local dependency patterns.

\textbf{Shifted-induction layout.} Designed to align the tree with induction-
style prefix-copy dependencies, where token $t$ in the second half copies token 
$t - N/2$ in the first half. Let $H = N/2$. We construct sibling pairs of the 
form $(\text{pair\_id},\ H + \text{pair\_id})$, placing each prefix token 
together with its corresponding copy in a shared low-level subtree. The pairs 
are then ordered according to a bit-reversal permutation to balance the tree 
depth across all positions. Formally, let $\{r_0, \dots, r_{H-1}\}$ be the 
bit-reversal of $\{0, \dots, H-1\}$ using $\lfloor \log_2 N \rfloor - 1$ bits. 
The leaf permutation is defined as
\begin{align*}
    \pi_{\mathrm{ind}}(2k) = r_k,\qquad
    \pi_{\mathrm{ind}}(2k+1) = H + r_k,
\end{align*}
where $k = 0, \dots, H-1$,
so that tree positions $(2k, 2k+1)$ are siblings corresponding to the original 
token pair $(r_k,\ H + r_k)$.

In both cases, the future mask used for autoregressive marginalization is derived 
directly from the permutation: tree position $p$ is treated as a future position 
at prediction step $t$ if and only if $\pi^{-1}(p) \geq t$.

\paragraph{Transformer}
The Transformer baseline is a decoder-only causal Transformer. It uses standard causal self-attention without any tree-structured attention mask. The model has 5 layers, 4 attention heads, a hidden size of 256, dropout 0.1, context length 32, and approximately 4.09M trainable parameters. It is trained with an autoregressive next-token cross-entropy objective.



\subsection{Experimental Details for Synthetic Data}
\label{app:synthetic_data}

\paragraph{Local-copy dataset.}
Each sequence of length $N$ is constructed by sampling $N/2$ tokens independently 
and uniformly from the vocabulary $\{1, \dots, V\}$, then repeating each token 
in the immediately following position:
$
    x = (a_1, a_1, a_2, a_2, \dots, a_{N/2}, a_{N/2}),
    \quad a_i \sim \mathrm{Uniform}(\{1,\dots,V\}).
$
The autoregressive target is the left-shifted sequence $y_t = x_{t+1}$, and the 
loss is computed over all positions. This dataset creates a purely local 
dependency structure: predicting $x_{t+1}$ requires only $x_t$.

\paragraph{Induction-style prefix-copy dataset.}
Each sequence of length $N$ is constructed by sampling a prefix of length $N/2$ independently and uniformly from $\{1, \dots, V\}$, then appending an exact copy:
$
    x = (a_1, \dots, a_{N/2}, a_1, \dots, a_{N/2}),
    \quad a_i \sim \mathrm{Uniform}(\{1,\dots,V\}).
$
The autoregressive target is the left-shifted sequence. To isolate the nonlocal induction behavior, the loss is computed only over the second half of the sequence: the targets for positions $1, \dots, N/2$ are masked out (set to the ignore index $-100$). Predicting $x_t$ for $t > N/2$ therefore requires associating position $t$ with its corresponding occurrence $t - N/2$ in the prefix, creating a nonlocal dependency of distance $N/2$.

\paragraph{Mixed-topology dataset.}
The mixed training distribution is constructed by combining the two datasets above with equal probability: each sample is drawn from the local-copy generator with probability $0.5$ and from the induction-style prefix-copy generator with probability $0.5$. For the local-copy samples within the mixed distribution, the loss is computed over all positions; for the induction-style samples, the loss is computed only over the second half, following the same masking convention as above. At evaluation time, we report performance separately on each constituent validation set to disentangle topology-specific behavior from average performance.

\subsection{Experimental Details for PTB and UD-ATIS}
\label{app:exp-details-real}

\paragraph{PTB.}
We use the Penn Treebank~\citep{marcus1993building} for autoregressive
language modeling with packed token segments of length 32 and a vocabulary of 10{,}000 tokens. We train 9 structure specialists and compare frozen specialist-initialized, unfrozen specialist-initialized, and from-scratch late mixtures. All specialists and mixtures use 8 channels, batch size 64, learning rate $10^{-4}$, and weight decay $10^{-4}$; specialists are trained for 30 epochs.

\paragraph{UD-ATIS.}
We use the UD English ATIS treebank for sentence-level autoregressive language modeling, filtering sentences to lengths between 4 and 32 tokens with a vocabulary of 2{,}000 tokens. We train 6 structure specialists and compare frozen specialist-initialized, unfrozen specialist-initialized, and from-scratch late mixtures. All specialists and mixtures use 4 channels, batch size 64, learning rate $10^{-4}$, and weight decay $10^{-4}$; models are trained for 3 epochs.

\paragraph{Structures.}
The structure denotes a fixed vtree layout over token positions. We consider several families of layouts: balanced trees, left- and right-branching chains, local block trees, strided or dilated chains, boundary-to-center or center-to-boundary chains, and half-pairing layouts. Local block layouts first compose nearby positions within contiguous blocks, while strided and dilated layouts group non-adjacent positions. The half-pairing layouts pair position $i$ with position $N/2+i$, which is intended to favor repeated or induction-like dependencies across the two halves of a sequence.

Each specialist uses a fixed vtree over token positions. PTB uses nine layouts covering left/right chains, local block trees, strided parity-based chains, outside-in ordering, and half-sequence pairing. UD-ATIS uses six specialist slots drawn from dilated chain, inside-out chain, right chain, and shifted induction. These structures provide different inductive biases over local, directional, long-range, and repeated-position dependencies.

\subsection{Hardware Configuration}\label{subsec:hardware}
All experiments are reproducible with single-process, single-GPU runs. Our hardware configuration consists of one NVIDIA RTX PRO 6000 Blackwell Server Edition GPU (98 GB memory) and an AMD EPYC 9555 CPU (256 logical cores, 1.0 TiB RAM).
\section{Prob-Transformer Parameterization}
\label{app:prob-transformer}

To isolate the effect of the output space independently from the context encoder, we construct a Prob-Transformer variant that replaces the standard logit-space output layer with a probability-space parameterization.

\paragraph{Model formulation.}
Given a prefix $\mathbf{x}_{\le t}$, the Transformer encoder produces a hidden representation
\[
\mathbf{h}_t \in \mathbb{R}^{d}.
\]
Instead of mapping $\mathbf{h}_t$ to logits, we directly parameterize the next-token distribution as a convex combination of basis distributions:
\[
p(x_{t+1} \mid \mathbf{x}_{\le t}) = \sum_{i=1}^{r} c_i(\mathbf{x}_{\le t}) \, B_i,
\]
where:
\begin{itemize}
    \item $c(\mathbf{x}_{\le t}) \in \Delta^r$ is a vector of mixture coefficients,
    \item $\{B_i\}_{i=1}^r$, with $B_i \in \Delta^{|V|}$, are learnable basis distributions over the vocabulary,
    \item $r$ is the mixture rank.
\end{itemize}

\paragraph{Mixture coefficient parameterization.}
We derive the mixture coefficients directly from the hidden representation via a positive normalization:
\[
c_i(\mathbf{x}_{\le t}) = \frac{\mathrm{softplus}(h_{t,i})}{\sum_{j=1}^{r} \mathrm{softplus}(h_{t,j})}.
\]
In our implementation, we set $r = d$, tying the mixture rank to the embedding dimension. This mirrors the Logit-HMM construction, where the latent state dimension serves as the feature space.

This formulation enforces that the predictive distribution lies in the convex hull of $\{B_i\}_{i=1}^r$, thereby introducing a probability-space output bottleneck analogous to that of PCs. Importantly, the context encoder remains unchanged, allowing us to isolate the effect of the output parameterization while keeping the representational capacity of the encoder fixed.

\section{Model Size}
\label{app:model-size}

\paragraph{Parameter scaling.}
Let $V$ denote the vocabulary size and $d$ the latent or embedding dimension. In our PTB and WikiText experiments, $V = 50{,}257$, so the dominant contribution to the parameter count comes from vocabulary-sized projections.

The parameter counts for the four model families scale approximately as:
\[
\text{HMM}(d) = d^2 + Vd + d \;\approx\; Vd,
\]
\[
\text{Logit-HMM}(d) = d^2 + 2Vd + V + d \;\approx\; 2Vd,
\]
\[
\text{Transformer}(d) \approx 2Vd + 24d^2 + O(Ld),
\]
\[
\text{Prob-Transformer}(d) \approx 2Vd + 24d^2 + O(Ld).
\]

\paragraph{Discussion.}
For the Transformer and Prob-Transformer, we use $n_{\text{layer}}=2$. The $24d^2$ term arises from the two Transformer blocks, while the $2Vd$ term comes from the input embedding and the output vocabulary projection. 

The Prob-Transformer introduces an additional $d \times V$ probabilistic basis, making its parameter count comparable to that of a Transformer without weight tying. Similarly, the Logit-HMM includes both an emission matrix and a logit projection, leading to two vocabulary-sized components. In contrast, the vanilla HMM contains only a single $d \times V$ emission matrix, making it roughly a factor of two smaller.

\paragraph{Concrete sizes.}
To make the scaling concrete, we instantiate the formulas using the PTB and WikiText configuration, where $|\mathcal{V}|=50{,}257$ and $d \in \{8,16,32,64,128\}$. The resulting parameter counts are:

\[
\begin{array}{c|ccccc}
d & 8 & 16 & 32 & 64 & 128 \\ \hline
\text{HMM} & 0.40\text{M} & 0.80\text{M} & 1.61\text{M} & 3.22\text{M} & 6.45\text{M} \\
\text{Logit-HMM} & 0.85\text{M} & 1.66\text{M} & 3.27\text{M} & 6.49\text{M} & 12.93\text{M} \\
\text{Transformer} & 0.81\text{M} & 1.62\text{M} & 3.24\text{M} & 6.54\text{M} & 13.27\text{M} \\
\text{Prob-Transformer} & 0.81\text{M} & 1.62\text{M} & 3.24\text{M} & 6.54\text{M} & 13.27\text{M}
\end{array}
\]

\paragraph{Takeaway.}
Except for the vanilla HMM, the other three models have broadly comparable parameter counts at the same $d$. This is because they all contain two vocabulary-sized $d \times V$ components, whereas the HMM contains only one. As a result, comparisons among Logit-HMM, Transformer, and Prob-Transformer isolate differences in output parameterization rather than model capacity.

\section{Training Details}
\label{app:training}

\paragraph{Overview.}
We evaluate our hypotheses across three datasets: \textbf{Simple Word}, \textbf{Penn Treebank (PTB)}, and \textbf{WikiText-103}. We compare four model families: \textbf{HMM}, \textbf{Logit-HMM}, \textbf{Transformer}, and \textbf{Prob-Transformer}. For Simple Word, we sweep the latent or embedding dimension over $d \in \{4,6,8,10\}$; for PTB and WikiText, we sweep $d \in \{8,16,32,64,128\}$. All Transformer-based models use a 2-layer architecture. We use 2 attention heads for Simple Word and 4 attention heads for PTB and WikiText.

To maintain consistency across experiments, we adopt standardized optimization parameters for each model family. We use the AdamW optimizer with a constant learning rate and the standard autoregressive training objective. The following table summarizes the primary optimization hyperparameters.

\begin{table}[h]
\centering
\caption{Common optimization hyperparameters across model families.}
\begin{tabular}{lcccc}
\toprule
\textbf{Hyperparameter} & \textbf{HMM} & \textbf{Logit-HMM} & \textbf{Transformer} & \textbf{Prob-Transformer} \\
\midrule
Learning Rate & $0.01$ & $5\times 10^{-4}$ & $1\times 10^{-4}$ & $5\times 10^{-4}$ \\
Weight Decay & $0.0$ & $5\times 10^{-6}$ & $0.1$ & $0.1$ \\
Dropout & $0.0$ & $0.0$ & $0.1$ & $0.1$ \\
\bottomrule
\end{tabular}
\end{table}

\paragraph{Dataset-Specific Configurations.}
Across all datasets, we use early stopping with a patience of 5 epochs and report results based on the best validation loss. We adjust the maximum number of epochs, batch size, sequence length, and attention heads according to dataset scale.

\paragraph{Simple Word.}
We use character-level tokenization for a last-token prediction task. The vocabulary size is $|\mathcal{V}|=54$, and we sweep $d \in \{4,6,8,10\}$. Transformer-based models use 2 layers and 2 attention heads. All models are trained with a maximum of 100 epochs, early stopping patience 5, and batch size 64.

\paragraph{PTB.}
We use GPT-2 tokenization and optimize the full-sequence autoregressive objective with sequence length $L=64$ on 17,736 training sequences. The vocabulary size is $|\mathcal{V}|=50{,}257$, and we sweep $d \in \{8,16,32,64,128\}$. Transformer-based models use 2 layers and 4 attention heads. All models are trained with a maximum of 200 epochs, early stopping patience 5, and batch size 32.

\paragraph{WikiText.}
We use GPT-2 tokenization with sequence length $L=128$ on a subset of 200K training sequences. The vocabulary size is $|\mathcal{V}|=50{,}257$, and we sweep $d \in \{8,16,32,64,128\}$. Transformer-based models use 2 layers and 4 attention heads. All models are trained with a maximum of 150 epochs, early stopping patience 5, and an effective batch size of 128.

\end{document}